\documentclass[]{IEEEtran}

\IEEEoverridecommandlockouts                         

\usepackage{graphics} 
\usepackage{iftex}
\ifPDFTeX
  \usepackage{times}
\else
  \usepackage{fontspec}
\fi
\usepackage{amsmath} 
\usepackage{amssymb}  
\usepackage{amsthm}
\usepackage{hyperref}
\hypersetup{pdftitle={nanoflann: A Header-Only KD-Tree Library for Exact
  Nearest-Neighbor Search on Compile-Time Product Manifolds},
  pdfauthor={Jose Luis Blanco-Claraco}}

\newtheorem{lemma}{Lemma}

\newcommand{\fyes}{\checkmark}
\newcommand{\fno}{--}
\newcommand{\fpart}{$\sim$}

\usepackage{import}
\usepackage{pgf}
\usepackage{tikz}
\usetikzlibrary{arrows.meta}

\usepackage{color}
\usepackage{xcolor}
\usepackage{listings}
\usepackage{caption}
\usepackage{subcaption}
\usepackage{algpseudocode}

\newcounter{nalg}
\renewcommand{\thenalg}{\arabic{nalg}} 
\DeclareCaptionLabelFormat{algocaption}{Algorithm \thenalg} 

\lstnewenvironment{algorithm}[1][] 
{   
    \refstepcounter{nalg} 
    \captionsetup{labelformat=algocaption,labelsep=colon} 
    \lstset{ 
		comment=[l]{\#},
        mathescape=true,
        frame=tB,
        numbers=left, 
        numberstyle=\tiny,
        basicstyle=\scriptsize, 
        commentstyle=\color{red},
        keywordstyle=\color{black}\bfseries\em,
        keywords={,input, output, return, datatype, function, in, if, else, foreach, while, begin, end, function, }
        numbers=left,
        xleftmargin=.04\textwidth,
        #1 
    }
}
{}

\title{\LARGE \bf
 nanoflann: A Header-Only KD-Tree Library for Exact Nearest-Neighbor Search on
 Compile-Time Product Manifolds
}

\author{Jose~Luis~Blanco-Claraco%
\thanks{J.L. Blanco-Claraco is with the Department of Engineering, CIESOL,
CeiA3. University of Almer\'{\i}a, 04120, Almer\'{\i}a, Spain
(e-mail: \mbox{jlblanco@ual.es}).}%
}

\begin{document}
\bstctlcite{IEEEexample:BSTcontrol}

\maketitle
\begin{tikzpicture}[remember picture,overlay]
\node[anchor=south,yshift=0.1in,text width=\textwidth,align=justify,
      font=\scriptsize] at (current page.south) {\copyright~2026 IEEE.
Personal use of this material is permitted. Permission from IEEE must be
obtained for all other uses, in any current or future media, including
reprinting/republishing this material for advertising or promotional purposes,
creating new collective works, for resale or redistribution to servers or
lists, or reuse of any copyrighted component of this work in other works.
Accepted for publication in IEEE Robotics and Automation Letters.
DOI: to be assigned};
\end{tikzpicture}
\thispagestyle{empty}
\pagestyle{empty}

\begin{abstract}
Nearest-neighbor (NN) search underlies countless robotics algorithms, but many
state spaces are \emph{not} Euclidean: orientations live on $SO(2)$ or $SO(3)$
and poses on $SE(2)$ or $SE(3)$, and a Euclidean KD-tree on their raw
coordinates returns \emph{wrong} neighbors. We present a header-only, dependency-free C++ KD-tree library for \emph{exact}
NN search on user-defined product manifolds declared at compile time: the space
is composed from base manifolds ($\mathbb{R}^n$, $SO(2)$, $SO(3)$, $S^n$;
products $SE(2)$, $SE(3)$, tori) with a single type alias, and the library
synthesizes a metric whose pruning is provably admissible, wrap-aware for
circular coordinates and double-cover-aware for unit quaternions. Experiments confirm exactness (recall $1.0$ vs an exhaustive scan
under the same product metric), where the Euclidean-on-raw shortcut collapses to recall
$\approx 0.5^{m}$ for $m$ rotation blocks, with up to three orders of magnitude
speedup over brute force. We further contribute an \emph{incremental, self-balancing single-tree index}
for points under heavy churn (e.g.\ sliding-window LiDAR maps) or $SE(3)$
trajectories: bounded memory and faster updates than prior work, with an
optional asynchronous tail-latency-bounding mode.
The implementation is open source: \url{https://github.com/jlblancoc/nanoflann}.
\end{abstract}

\section{INTRODUCTION}
\label{sec:intro}

Nearest-neighbor (NN) search is a foundational primitive in mobile robotics: it
underlies the data association of point-cloud registration and LiDAR odometry,
the nearest-keyframe and loop-closure lookups of SLAM, and the steering queries
of sampling-based motion planners. Issued millions of times per second, these
queries are almost always answered by a $k$-dimensional (KD)
tree~\cite{bentley1975kdtree,friedman1977kdtree,samet2006foundations}, which is
frequently the performance bottleneck of the whole system.

A recurring but under-served need is that many robotics quantities do \emph{not}
live in a Euclidean space: a heading is a point on the circle $S^1\cong SO(2)$,
a 3D orientation an element of $SO(3)$, a rigid-body pose lives on $SE(2)$ or
$SE(3)$. The shortest distance between two such elements is a geodesic, and the
spaces wrap around (the angles $+\pi$ and $-\pi$ coincide) or double-cover (the
unit quaternions $\mathbf{q}$ and $-\mathbf{q}$ denote the same rotation). Feeding their raw
coordinates to a Euclidean KD-tree therefore returns \emph{wrong} neighbors:
the metric is wrong and, what is worse, the tree's branch-and-bound pruning
silently discards true neighbors that lie across a wrap-around seam.

Adapting KD-trees to such topologies is itself not new: Yershova and
LaValle~\cite{yershova2007mpnn} showed that an axis-aligned tree can answer
\emph{exact} geodesic NN on products of $\mathbb{R}$, $S^1$, and $\mathbb{R}P^3$
by making the point-to-region bound topology-aware, and later work refined this
for $SO(3)$/$SE(3)$~\cite{ichnowski2015se3,ichnowski2020concurrent}. These
results live in planning-specific libraries that hard-code a fixed set of spaces
and store topological state per node. Our goal is different: to bring exact
manifold NN to a mainstream, general-purpose KD-tree as a \emph{drop-in}
capability. The user declares the search space \emph{topology} at compile time,
as a product of arbitrary base manifolds, and we
synthesize from that declaration a distance functor whose pruning bound is
\emph{admissible} for each coordinate's topology, i.e.\ it never
over-estimates the distance from the query to a subtree's region, and therefore
never prunes a subtree that could hold a true neighbor (formalized and proved in
Lemma~\ref{lem:adm}). Admissibility is what makes branch-and-bound
exact, so the search returns the same neighbors as an exhaustive scan
under the same metric.
Table~\ref{tab:features} positions this work against the existing NN libraries and
methods, including those without a known public reference implementation.

\begin{table*}[t]
\centering
\caption{Feature comparison of NN libraries and methods (\fyes~yes,
\fpart~partial, \fno~no). \emph{Non-Eucl.\ spaces:} ``any'' = arbitrary metric
as a black box; ``fixed'' = a hard-coded factor set chosen at run time
(MPNN: $\mathbb{R}$, $S^1$, $\mathbb{R}P^3$; nigh: $\mathbb{R}^n$, $SO(2)$,
$SO(3)$, $SE(3)$, weighted Cartesian products); ``general'' =
arbitrarily nested compile-time products of base manifolds, adding
$S^n$ ($n\ge2$). The proposed library supports the
manifold metric in both the static and the incremental single-tree index.}
\label{tab:features}
\scriptsize
\renewcommand{\arraystretch}{1.0}
\begin{tabular}{l c c c c c}
\hline
Method & Exact & Non-Eucl.\ spaces & Header-only & Dep.-free & Dynamic\\
\hline
\multicolumn{6}{l}{\textit{Euclidean KD-tree / NN libraries}}\\
FLANN~\cite{muja2009flann} & approx/exact & none & \fno & \fno & \fno\\
ANN~\cite{arya1998optimal} & exact$+\epsilon$ & none & \fno & \fno & \fno\\
libnabo~\cite{elseberg2012libnabo}, PCL~\cite{rusu2011pcl} & \fyes & none & \fno & \fno & \fno\\
pico\_tree~\cite{picotree}, libkdtree++ & \fyes & none & \fyes & \fyes & \fno\\
ikd-Tree~\cite{cai2021ikdtree} & \fyes & none & \fyes & \fpart & \fyes\\
\hline
\multicolumn{6}{l}{\textit{General-metric (manifold as a black box) and approximate}}\\
VP-tree~\cite{yianilos1993vptree} / cover tree~\cite{beygelzimer2006cover} & \fyes & any & \fno & \fno & \fpart\\
HNSW~\cite{malkov2018hnsw} & approx & any & \fyes & \fyes & \fyes\\
\hline
\multicolumn{6}{l}{\textit{Exact manifold KD-trees}}\\
MPNN~\cite{yershova2007mpnn} & \fyes & fixed & \fno & \fno & \fyes\\
Ichnowski'15~\cite{ichnowski2015se3} & \fyes & fixed & \fpart & \fno & \fyes\\
nigh~\cite{ichnowski2020concurrent} & \fyes & fixed & \fyes & \fno & \fyes\\
\textbf{proposed} & \fyes & \textbf{general} & \fyes & \fyes & \fyes\\
\hline
\end{tabular}
\end{table*}

The contributions of this work are:
\begin{itemize}
\item A \emph{compile-time, composable} product-manifold abstraction for
  KD-trees: base spaces ($\mathbb{R}^n$, $SO(2)$, $SO(3)$, the unit sphere
  $S^n$) and arbitrary products ($SE(2)$, $SE(3)$, $n$-tori, $SO(3)^k$,
  $\mathbb{R}^3\times S^2$ for oriented points, $\ldots$) declared with a
  single C++ type alias, generalizing the fixed space sets of prior work.
\item A \emph{coordinate-separable} admissible bound (wrap-aware for circular
  coordinates, double-cover-aware for quaternion blocks) that keeps the
  manifold tree's node layout and memory \emph{identical} to the ordinary
  KD-tree, unlike prior schemes that store per-node box projections. It
  suffices for all linear and circular coordinates ($\mathbb{R}^n$, $SO(2)$,
  $S^n$, $SE(2)$, tori) and enables the $O(1)$ per-level update. For quaternion
  blocks it is too loose, and we reuse MPNN's block-level chordal
  bound~\cite{yershova2007mpnn} \emph{(not our contribution)}; ours is its
  \emph{storage-free realization}, reconstructing the block's intervals from the
  split scalars the tree already stores. Exactness is proved for any weighted
  product metric with a per-block-monotone combination (Lemma~\ref{lem:adm}).
\item The adapted query algorithms (Sec.~\ref{sec:manifold}), including a
  region-carrying exact search that \emph{reconstructs} child intervals from the
  stored split values, handling the circle's $\pm\pi$ seam with no extra
  per-node storage.
\item An open-source implementation, an evaluation of exactness and of
  \emph{when} manifold KD-trees beat brute force, and a loop-closure case study
  on real trajectories where the common workarounds lose precisely the
  loop-closure candidates.
\item An \emph{incremental self-balancing single-tree index}
  (Sec.~\ref{sec:incremental}) for point sets under heavy churn: a from-scratch
  re-implementation of the published ikd-Tree algorithms~\cite{cai2021ikdtree}
  with weight-balanced partial rebuilds, box-trim deletion, bounded memory and
  an optional asynchronous rebalancing thread; its per-subtree boxes also supply
  the region supersets making it exact on product manifolds under churn.
\end{itemize}

\noindent Composing metrics across a product is \emph{not} itself new (MPNN and nigh
do it); the novelty is in the storage-free bound above, the extra base spaces
($S^n$, hence $\mathbb{R}^3\times S^2$), and in delivering both inside a
general-purpose KD-tree. An alternative Bentley--Saxe forest~\cite{bentley_decomposable_1980} implementation is also provided for comparison, covering Euclidean-only spaces.

\section{RELATED WORK}
\label{sec:related}

\paragraph*{KD-tree and NN libraries}
The KD-tree~\cite{bentley1975kdtree,friedman1977kdtree} and its branch-and-bound
search remain the workhorse for low-dimensional exact NN~\cite{samet2006foundations}.
FLANN~\cite{muja2009flann} popularized \emph{approximate} NN for high-dimensional
data, ANN~\cite{arya1998optimal} provides exact and $\epsilon$-approximate
search with provable bounds, and in robotics libnabo~\cite{elseberg2012libnabo},
PCL~\cite{rusu2011pcl} (a thin layer over FLANN) and
Open3D~\cite{zhou2018open3d} serve point clouds. All operate on Euclidean
coordinates.
For data that changes over time, the Bentley--Saxe
transform~\cite{bentley_decomposable_1980} makes a static structure
insertion-friendly at the cost of querying $O(\log N)$ sub-trees and of
unbounded memory under deletion churn, while the
ikd-Tree~\cite{cai2021ikdtree}, an incrementally rebalanced single tree with
lazy deletion and box-wise trimming, underpins LiDAR-inertial odometry in works such as FAST-LIO2~\cite{xu2022fastlio2}; this work presents a re-implementation of its published design from scratch (Sec.~\ref{sec:incremental}),
including further optimizations, and compares it against the original in
Sec.~\ref{sec:exp-incremental}.

\paragraph*{General-metric data structures}
VP-trees~\cite{yianilos1993vptree} and cover trees~\cite{beygelzimer2006cover}
answer exact NN in \emph{any} metric space by treating the metric as a black
box, but every pruning test then needs a full metric evaluation, with no
coordinate-separable incremental bound; on the low-dimensional spaces of
robotics they are typically slower than an axis-aligned KD-tree by orders of
magnitude~\cite{yershova2007mpnn}.

\paragraph*{Exact NN on manifolds with KD-trees}
Adapting axis-aligned KD-trees to the topologies of robotics configuration
spaces originates with Yershova and LaValle~\cite{yershova2007mpnn}: their MPNN
library builds an ANN-style tree on the raw coordinates and, at query time,
uses topology-aware point-to-rectangle distances on Cartesian products of
$\mathbb{R}$, $S^1$, and $\mathbb{R}P^3$ (i.e.\ $SE(2)$, $SE(3)$, tori,
$SO(3)^k$), handling the quaternion double cover with the per-node block bound
$\min(d(p,R),d(-p,R))$ and proving exact pruning. Ichnowski and
Alterovitz~\cite{ichnowski2015se3} later gave an intrinsic $SO(3)$ scheme that
partitions $S^3$ with four KD-trees over projected quaternion faces, extended
to $SE(3)$ and to concurrent search (nigh)~\cite{ichnowski2020concurrent}; both
ship as planning-specific libraries that hard-code their space sets and store
per-node topological state. Exact manifold KD-tree NN and the composition of per-factor metrics over a
product are both established, and the block-level quaternion bound we use is due
to~\cite{yershova2007mpnn}. Three deltas remain: (i) that bound is applied with
\emph{no} per-node storage, reconstructing the block's intervals from the split
scalars the tree already keeps; (ii) the factor set is open and resolved at
compile time, adding $S^n$ ($n\ge2$) and products containing it, $n$-tori and
$SO(3)^k$; (iii) it ships in a general-purpose dependency-free KD-tree and an
incremental index, not a planning-specific library. Sampling-based
planners~\cite{lavalle2001rrt,karaman2011rrtstar} provide the motivating use
cases for non-Euclidean nearest-neighbor queries.

\section{BACKGROUND: EUCLIDEAN KD-TREE}
\label{sec:background}

We recap the standard construction and search on which the non-Euclidean
extension builds, following the proposed library's implementation.

\subsection{Construction}
$N$ points in $\mathbb{R}^D$ are indexed by a permutation array and the tree is
built top-down, splitting each node by an axis-aligned plane (an in-place
partition of that array) until a leaf holds at most \texttt{maxLeafSize} points.
The split dimension is the one of largest spread; the cut is the midpoint of the
node's box along it, clamped so neither child is empty. An interior node stores
only the split dimension and two scalars \texttt{divLow}, \texttt{divHigh}, the
two edges of the empty slab separating the children; no per-node bounding box is
kept, only the root box. These suffice both to bound the distance to a pruned
subtree and, as Sec.~\ref{sec:manifold} shows, to \emph{reconstruct} the
per-coordinate node regions exact manifold search requires.

\subsection{Query with incremental distance bound}
\label{sec:euclidean-search}

A $k$-NN query is a branch-and-bound descent. At an interior node with split
dimension $c$ it visits the near child first, and the far child only if a lower
bound $B$ on the squared distance from the query $v$ to that child's region is
below the pruning threshold $\tau$ (the $k$-th smallest squared distance so
far). $B$ is maintained \emph{incrementally}: the search keeps each dimension's
contribution to the distance from $v$ to the current region (initialized from
the root box) and $B$ is their running sum, so entering the far child updates
only the split dimension's term, $B \leftarrow B + g_c - \rho_c$, in $O(1)$
rather than recomputing a $D$-dimensional distance. Correctness rests on $g_c$
being an \emph{admissible} lower bound of the distance from $v$ to the far
region along $c$; in Euclidean space it is simply $(v_c-s_c)^2$, the squared gap
to the split value $s_c$. Generalizing this single quantity to non-Euclidean
topologies is the core contribution of this work.

\section{EXACT NN ON PRODUCT MANIFOLDS}
\label{sec:manifold}

The user declares the search-space topology \emph{at compile time} as a product
manifold $M = M_1\times\cdots\times M_b$ of base spaces ($\mathbb{R}^n$,
$SO(2)$, $SO(3)$ as unit quaternions, the sphere $S^n$; products yield $SE(2)$,
$SE(3)$, the $n$-torus, $SO(3)^2$, $\mathbb{R}^3\times S^2$), and we synthesize a
distance functor making the tree of Sec.~\ref{sec:background} honor that metric
while keeping the search \emph{exact}.
The declaration is flattened at compile time into a table mapping each
scalar coordinate to its topology tag, \textsc{Linear}, \textsc{Circular}, or
\textsc{Quaternion}; this table is the only thing the metric and the search
consult, and it costs no run-time storage. Angles are assumed normalized to
$[-\pi,\pi]$. Unit-norm quaternions are assumed wherever a chordal distance is
read back as a rotation angle; the bounds and the proof below never use the
norm, so a uniformly scaled block (Sec.~\ref{sec:exp-app}) is equally valid.

\subsection{Product metric}
We equip $M$ with the weighted squared product metric
$d^2(p,q) = \sum_{j=1}^{b} w_j^2\, d^2_{M_j}(p_j,q_j)$, $w_j>0$,
with per-base contributions
\begin{equation}
d^2_{M_j}(p_j,q_j) =
\begin{cases}
\sum_i (p_i-q_i)^2 & M_j=\mathbb{R}^n,\\[2pt]
\mathrm{wrap}(p-q)^2 & M_j=SO(2),\\[2pt]
\min\!\big(\lVert p-q\rVert^2,\ \lVert p+q\rVert^2\big) & M_j=SO(3),
\end{cases}
\label{eq:metric}
\end{equation}
where $\mathrm{wrap}(\cdot)$ maps an angle difference into $[-\pi,\pi]$. For
$SO(3)$ we use the \emph{chordal} distance on the unit quaternion sphere
$S^3\subset\mathbb{R}^4$. For unit $p,q$ it satisfies the elementary
identity
$\min(\lVert p-q\rVert^2,\lVert p+q\rVert^2)=2-2|\langle p,q\rangle|
=4\sin^2(\theta/4)$ with $\theta=2\arccos|\langle p,q\rangle|$ the geodesic
angle, so it is a strictly increasing function of $\theta$ on $[0,\pi]$ and
induces the same nearest-neighbor ranking; it is the standard quaternion
chordal metric, the quaternion form of the chordal rotation
distance~\cite{hartley2013rotation}, while the $\min$
accounts for the antipodal double cover $q\sim -q$. Being the ambient Euclidean metric of $\mathbb{R}^4$ restricted to $S^3$, it
keeps axis-aligned splits and per-coordinate interval bounds valid: the sphere
constraint is never used in the bound, which only makes it conservative (a true
lower bound), never inadmissible. This is what lets a \emph{linear} KD-tree serve
$SO(3)$ exactly. This is also why we index quaternions rather than rotation matrices: the
matrix chordal distance $\lVert R_1-R_2\rVert_F=2\sqrt{2}\sin(\theta/2)$
\cite{hartley2013rotation} is a fixed monotone function of the quaternion one,
so both induce the \emph{same} ranking, but a matrix tree would spread a
$3$-dimensional manifold over $9$ axes instead of $4$, weakening every
axis-aligned bound. Callers holding matrices convert once, in $O(1)$.

The unit sphere $S^n$ (3D directions for $n=2$) follows the same embedding:
unit vectors with chordal distance $\lVert p-q\rVert^2=4\sin^2(\theta/2)$,
monotone in the geodesic angle $\theta$ and \emph{literally} the ambient
$\mathbb{R}^{n+1}$ metric, so $S^n$ coordinates reduce to \textsc{Linear} ones
with no new bound and, unlike $SO(3)=\mathbb{R}P^3$, no double cover. Only
$S^1=SO(2)$ is handled \emph{intrinsically} (one wrapped angle), since it alone
admits a global separable chart.

\subsection{Exact pruning via an admissible coordinate bound}
\label{sec:bound}
A subtree whose region is the axis-aligned box $R=\prod_i I_i$
($I_i=[\ell_i,h_i]$) is pruned when an incrementally maintained lower bound on
the distance from the query to $R$ exceeds the current $k$-th best distance.
Correctness thus reduces to supplying, per coordinate, an \emph{admissible}
(never over-estimating) lower bound of its distance contribution. We define
\begin{equation}
g_i(v,I_i)=
\begin{cases}
\big(\max(0,\ell_i-v,v-h_i)\big)^2 & \text{linear},\\[2pt]
\big(\max(0,\,\delta(v,I_i))\big)^2 & \text{circular},\\[2pt]
\min\big(b^{+}_i,\,b^{-}_i\big) & \text{quaternion},
\end{cases}
\label{eq:bound}
\end{equation}
where, for a circular coordinate, $\delta(v,I_i)=0$ if $v\in I_i$ and otherwise
the smaller of the two \emph{wrapped} distances from $v$ to the endpoints
$\ell_i,h_i$; and for a quaternion coordinate $b^{+}_i,b^{-}_i$ are the linear
bounds of $v$ and $-v$ against $I_i$. The whole region bound is the running sum
$\sum_i g_i$, updated one coordinate at a time exactly as in the Euclidean case.

\begin{lemma}[Admissibility]
\label{lem:adm}
For every node region $R=\prod_i I_i$ that contains all data points of its
subtree, and every query $v$,
$\sum_i w^2_{j(i)}\, g_i(v,I_i) \le \min_{p\in R\cap \mathrm{Data}} d^2(v,p)$.
Consequently a subtree is pruned only when it cannot contain a closer neighbor,
i.e.\ the search returns the exact nearest neighbors.
\end{lemma}

\begin{proof}
For linear and circular coordinates the metric \eqref{eq:metric} is separable,
$d^2(v,p)=\sum_i f_i(v_i,p_i)$ with $f_i\ge 0$. The $SO(3)$ block is \emph{not}
separable, since its chordal distance is $\min(\sum_i a_i,\sum_i b_i)$ with
$a_i=(v_i-p_i)^2$, $b_i=(v_i+p_i)^2$; we instead use the separable \emph{lower
bound} $f_i=\min(a_i,b_i)$, which satisfies $\sum_i f_i\le\min(\sum_i a_i,\sum_i
b_i)=d^2_{SO(3)}(v,p)$ because $\sum_i\min(a_i,b_i)\le\min(\sum_i a_i,\sum_i b_i)$.
Pruning by any $g\le d^2$ is admissible, so in all cases it suffices to show
$g_i(v,I_i)\le f_i(v,p_i)$ for every $p_i\in I_i$ and then sum. For a linear coordinate this is
the classical point-to-interval inequality. For a circular coordinate, $I_i$ is a
direct arc that contains $p_i$; the nearest point of a direct arc to an external
query is an endpoint, and using the \emph{wrapped} endpoint distance never
over-estimates, hence $g_i\le f_i$; if $v\in I_i$ then $g_i=0\le f_i$. For a
quaternion coordinate, $\min(b^+_i,b^-_i)$ lower-bounds each of the two covers
coordinate-wise, and summing the minima lower-bounds the block $\min$ as noted
above. As $R\supseteq R\cap\mathrm{Data}$, the bound also holds against the data
in the subtree.
Two generalizations follow at once. (i) \emph{Weights:} multiplying the
$j$-th block inequality by $w_j^2>0$ preserves it; since
$\min(\lVert wp-wq\rVert^2,\lVert wp+wq\rVert^2)=w^2\min(\lVert p-q\rVert^2,
\lVert p+q\rVert^2)$ and $-(wq)=w(-q)$, a weight may equivalently be applied in
the metric or by scaling the block's coordinates. (ii) \emph{Combination rule:}
the proof uses only that the per-block terms are combined by a function
non-decreasing in each argument, so any such rule works, including an $L_1$
(Manhattan) sum of unsquared \emph{per-block} terms ($\sqrt{G_B}$ below for a
quaternion or $S^n$ block, not a sum of coordinate gaps). A strictly increasing
reparameterization of a block bound also stays admissible, e.g.\
$\theta\ge4\arcsin(\sqrt{g_B}/2)$ bounds the geodesic angle, so a geodesic
rotation factor needs no new theory.
\end{proof}

\emph{Remark (block-level tightening).} For a quaternion block $B$ (four
coordinates) the per-coordinate sum is loose, since
$\sum_{i\in B} \min(a_i,b_i)\le \min(\sum_i a_i,\sum_i b_i)$. Because the
region-carrying search (Sec.~\ref{sec:algos}) already holds the block's four
intervals $R_B=\prod_{i\in B}I_i$, the prune test recomputes in $O(1)$ the
\emph{tight} block bound
$G_B=\min\!\big(d^2_{\mathbb{R}^4}(v,R_B),\,d^2_{\mathbb{R}^4}(-v,R_B)\big)$~\cite{yershova2007mpnn}
from them (Fig.~\ref{fig:circle-schematic}(b)), matching MPNN's per-node block bound while storing \emph{no} per-node
box (it is reconstructed, not stored) and leaving the Euclidean update and node
layout unchanged.
How $G_B$ enters matters, as the running sum stays strictly
per-coordinate: $\mathrm{dists}[\,]$ and $B_{\mathrm{lo}}=\sum_i g_i$ are
maintained as in the Euclidean case ($O(1)$ per level), and $G_B$ enters only at
the prune \emph{gate}, a non-negative correction replacing each block's loose
share by its tight one,
\begin{equation}
B_{\mathrm{gate}} \;=\; B_{\mathrm{lo}} \;+\;
  \sum_{B}\Big(\,G_B \;-\; \textstyle\sum_{i\in B} g_i\,\Big) \;\ge\; B_{\mathrm{lo}},
\label{eq:gate}
\end{equation}
visiting the far child iff $B_{\mathrm{gate}}\le\tau$. Each parenthesized term
is $\ge 0$, so \eqref{eq:gate} stays a lower bound and Lemma~\ref{lem:adm}
applies unchanged; the \emph{loose} $B_{\mathrm{lo}}$ is what descends, so no
state is rewritten and each deeper node re-tightens its own gate, at four
interval evaluations per block.
The gain on rotation-heavy spaces is quantified in
Sec.~\ref{sec:exp-manifold}.

\subsection{Algorithms adapted to non-Euclidean groups}
\label{sec:algos}

Three building blocks of Sec.~\ref{sec:background} change: the leaf distance,
the per-coordinate bound, and the descent.

The \emph{leaf distance} walks the topology table and accumulates each
coordinate's contribution from \eqref{eq:metric}, processing a quaternion block
jointly so the antipodal cover is exact rather than coordinate-wise. The
\emph{bound} $g_i$ implements \eqref{eq:bound} per topology
(\texttt{pointToIntervalDist} in Algorithm~\ref{alg8}).

The \emph{descent} needs one change specific to the circle: one cut value does
not suffice for \eqref{eq:bound}, since a far child wrapping across $\pm\pi$ can
be closer than its near edge suggests, so the Euclidean single-edge bound
over-prunes (Fig.~\ref{fig:circle-schematic}). The adapted search carries the
running per-coordinate region down the recursion, splitting it at the node's
slab, and evaluates $g_i$ against the \emph{full} child interval; near/far
choice and far bound both use $g_i$, so the descent is uniform across
topologies. For linear coordinates the
full-interval bound equals the Euclidean single-edge value, so the only added
cost is an $O(1)$ save/restore per level, incurred solely by manifold metrics.
Since the root box is a real interval and every split yields real
sub-intervals, each carried interval is a \emph{non-wrapping} arc, licensing the
containment test in \eqref{eq:bound}. MPNN \emph{stores} the box per node; we
\emph{reconstruct} it from the split values, and the same region supplies each
quaternion block's four intervals for the tight bound of Sec.~\ref{sec:bound}.

\begin{algorithm}[caption={manifold search, region-carrying (searchLevelManifold)}, label={alg8}, float]
 function: mfSearch()
 input: node, query v, mindist, dists[], region[]
 begin
   if (node is leaf)
     foreach p in node: addResult(p, dist(v,p))
     return
   end
   c $\gets$ node.divDim ; val $\gets$ v[c] ; cur $\gets$ region[c]
   # Child intervals, rebuilt from the split scalars:
   lo $\gets$ [cur.lo, divLow] ; hi $\gets$ [divHigh, cur.hi]
   dLo $\gets$ pointToIntervalDist(val, lo, topo[c])
   dHi $\gets$ pointToIntervalDist(val, hi, topo[c])
   if (dLo<=dHi) near,far,nR,fR $\gets$ left,right,lo,hi
   else          near,far,nR,fR $\gets$ right,left,hi,lo
   end
   region[c] $\gets$ nR
   mfSearch(near, v, mindist, dists, region)
   cutd $\gets$ pointToIntervalDist(val, fR, topo[c])
   mindist $\gets$ mindist + cutd - dists[c]
   save $\gets$ dists[c]
   dists[c] $\gets$ cutd ; region[c] $\gets$ fR
   # Gate only (Eq. 3): tighten the quaternion blocks.
   gate $\gets$ mindist
   foreach quaternion block B:
     gate $+{=}$ blockBound(v,region,B) $-$ $\sum_{i \in B}$dists[i]
   if (gate <= worstDist)
     mfSearch(far, v, mindist, dists, region)
   dists[c] $\gets$ save ;  region[c] $\gets$ cur
 end
\end{algorithm}

\begin{figure}[t]
\centering
\begin{subfigure}[b]{\columnwidth}
\centering
\import{figs}{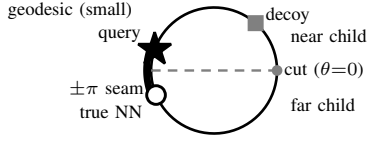}
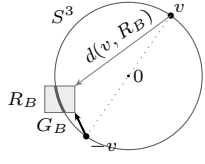
\caption{circle $SO(2)$}
\label{fig:sch-so2}
\end{subfigure}

\vspace{2pt}
\begin{subfigure}[b]{\columnwidth}
\centering
\begin{tikzpicture}[scale=0.78,every node/.style={font=\scriptsize},
                    >={Stealth[length=3pt]}]
  \draw[black!65] (0,0) circle (1.25);
  \node at (137:1.52) {$S^3$};
  \fill (0,0) circle (0.7pt);
  \node[inner sep=1.5pt,right=0pt] at (0,0) {$0$};
  \draw[dotted,black!55] (0.717,1.024) -- (-0.717,-1.024);
  \fill (0.717,1.024) circle (1.5pt);
  \node[above right,inner sep=1pt] at (0.717,1.024) {$v$};
  \fill (-0.717,-1.024) circle (1.5pt);
  \node[below right=-1pt,inner sep=2pt] at (-0.717,-1.024) {$-v$};
  \draw[fill=black!7,draw=black!55] (-1.42,-0.62) rectangle (-0.91,-0.17);
  \draw[line width=1.1pt,black!65] (187.9:1.25) arc (187.9:209.7:1.25);
  \node[left,inner sep=2pt] at (-1.42,-0.40) {$R_B$};
  \draw[->,black!50] (0.717,1.024)
       -- node[sloped,above,pos=0.46,inner sep=1.5pt,text=black] {$d(v,R_B)$}
       (-0.91,-0.17);
  \draw[->,thick] (-0.717,-1.024) -- node[left=1pt,pos=0.45] {$G_B$} (-0.91,-0.62);
\end{tikzpicture}
\caption{quaternion block $B\subset S^3$}
\label{fig:sch-quat}
\end{subfigure}
\caption{Why a Euclidean KD-tree mis-prunes on non-Euclidean spaces.
(a) On $SO(2)$ the query and its true nearest neighbor lie a
short geodesic arc apart across the $\pm\pi$ seam but at opposite ends of the
split coordinate, so the Euclidean point-to-cut bound prunes the far child,
whereas the wrap-aware bound \eqref{eq:bound} measures the wrapped distance and
explores it. (b) A quaternion block, on a cross-section of
$S^3$: the node box $R_B$ covers the arc its subtree occupies and is far from
$v$ but close to the antipode $-v$, the \emph{same} rotation, so pruning on
$d(v,R_B)$ alone would discard it. The bound is the smaller of the two ambient
point-to-box distances, $G_B=\min(d^2(v,R_B),d^2(-v,R_B))$, reflecting the
\emph{query}, never the box; it is
admissible because every $p\in R_B$ has $d^2(v,p)\ge d^2(v,R_B)$ and
$d^2(-v,p)\ge d^2(-v,R_B)$. Products follow blockwise, and for $S^n$ ($n\ge2$), its chordal metric \emph{is} the ambient one.}
\label{fig:circle-schematic}
\end{figure}

\section{INCREMENTAL SELF-BALANCING INDEX}
\label{sec:incremental}

Real-time LiDAR mapping needs a KD-tree over a point set that changes every
frame: insert the new scan, delete what leaves the window, query heavily in
between. Rebuilding the static tree each frame scales with the whole map, and
the Bentley--Saxe forest~\cite{bentley_decomposable_1980} multiplies query cost
by the number of sub-trees and never reclaims deleted points. Following the
design pioneered by the ikd-Tree~\cite{cai2021ikdtree}, the incremental index
keeps a \emph{single} self-balancing tree with one point per node, augmented
with subtree size, tombstone count and a subtree AABB. Inserts attach in
amortized $O(\log N)$; deletes are lazy tombstones; the AABB lets
\texttt{removeOutsideBox} drop whole subtrees in $O(1)$ as the window slides.
Balance comes from weight-balanced (BB[$\alpha$]) \emph{partial rebuilds},
triggered on the insertion unwind or when a subtree's tombstone fraction exceeds
a threshold; rebuilds reclaim tombstoned nodes through a pooled allocator with a
free-list, bounding steady-state memory under churn. The core is synchronous; an optional wrapper offloads large near-root
rebuilds to a background thread, bounding the update-latency tail at the price
of temporarily less-balanced queries and roughly twice the memory (as in
ikd-Tree).
Exactness under asynchrony needs no locking on the query path, since the
background thread never mutates the queried tree: a rebuild snapshots the live
indices and builds a \emph{separate} tree off-thread, while each foreground
insertion or deletion is applied to the current tree \emph{and} appended to an
operation log, which the foreground replays onto the fresh tree before swapping
it in. Queries thus always see one complete index whose subtree AABBs are
supersets of their contents, the sole precondition of Lemma~\ref{lem:adm}; a
partially built tree is unreachable, so exactness holds for every interleaving,
not merely the tested ones, at the cost of two trees during a rebuild.

The bound of Sec.~\ref{sec:bound} carries over at no extra cost: each
per-subtree AABB already maintained is exactly the raw-coordinate region $R$
Lemma~\ref{lem:adm} requires. Evaluating \eqref{eq:bound} against each child's
box, in place of the Euclidean cut distance, makes the incremental index exact
on the product manifolds of Sec.~\ref{sec:manifold} too, under churn and
background rebuilds, enabling growing $SE(3)$ keyframe databases (loop-closure
retrieval, RRT$^\ast$).

\section{EXPERIMENTAL RESULTS}

The evaluation covers Euclidean performance (Sec.~\ref{sec:exp-euclid});
exactness where the Euclidean shortcut fails (Sec.~\ref{sec:exp-manifold});
comparison with prior exact-manifold
libraries~\cite{yershova2007mpnn,ichnowski2020concurrent}
(Sec.~\ref{sec:exp-headtohead}); loop-closure retrieval on real $SE(3)$
trajectories (Sec.~\ref{sec:exp-app}); and updates under sliding-window churn
(Sec.~\ref{sec:exp-incremental}).

\emph{Setup.} AMD Ryzen 7 2700X (8C/16T, 3.7\,GHz), 48\,GB RAM, Ubuntu
24.04, GCC 13.3, \texttt{-O3 -march=native}. Timings are single-threaded except
where a method's background thread is named (thread counts are then stated);
every index uses maximum point count per leaf node of $10$; each time is a mean over $2000$
queries after warm-up, ${<}3\%$ spread over $5$ repetitions. Data:
KITTI~\cite{geiger2012kitti}, MulRan~\cite{kim2020mulran}; baselines
(pico\_tree~\cite{picotree}, FLANN, libnabo, PCL, HNSW, MPNN, nigh, ikd-Tree)
pinned to the commits listed in the benchmark repository,
\url{https://github.com/MRPT/nanoflann-benchmark}.

\subsection{Euclidean baseline and library comparison}
\label{sec:exp-euclid}
We first confirm the library is a competitive, drop-in Euclidean KD-tree,
benchmarking build and exact $1$-NN query on KITTI~\cite{geiger2012kitti}
Velodyne scans against seven widely used C++ implementations
(FLANN~\cite{muja2009flann} plain and randomized, PCL~\cite{rusu2011pcl},
libnabo~\cite{elseberg2012libnabo}, pico\_tree~\cite{picotree}, libkdtree++,
brute force).

On 3D LiDAR data the index builds and queries on par with the fastest across
four orders of magnitude of $N$: at
$N=2\times 10^6$ it answers exact queries in $1.17\,\mu$s (pico\_tree $1.08$,
FLANN $1.18$) and builds in $0.22$\,s (pico\_tree $0.19$\,s) versus $3.6$\,ms
per brute-force query, while header-only and dependency-free. pico\_tree's small
build edge reflects its lighter box-midpoint split versus the proposed
spread-based axis choice, a build-time-versus-split-tightness trade-off. The
high-dimensional regime inverts this: at $d=128$ every exact KD-tree, ours
included ($1.9$\,ms/query), approaches brute force ($7.3$\,ms) and only
approximate methods stay fast ($46$--$169\,\mu$s at recall $\approx 0.99$), so
exact KD-trees suit low-dimensional state spaces, not descriptor matching.

\subsection{Exact nearest neighbors on non-Euclidean spaces}
\label{sec:exp-manifold}

We validate the index on $SO(2)$, $S^2$, $SO(3)$, $SE(3)$ and $SO(3)^2$, with
uniformly random data and queries (angles, unit vectors/quaternions,
translations in a $\pm 10$ box). Quaternions are
drawn uniformly on $S^3$ and stored with \emph{random} (uncanonicalized) signs,
the worst case for the double cover. The ground truth is an exhaustive search using the
same product metric \eqref{eq:metric}, whose rotational factor is chordal
(Sec.~\ref{sec:manifold}); for the pure-rotation spaces this ordering coincides
with the geodesic one; $k=1$ throughout.

\emph{Correctness.} For every space and size up to $N=10^6$ the index returns
distances identical to brute force (recall $=1.0$, machine precision),
confirming Lemma~\ref{lem:adm}. Feeding the same raw coordinates to a Euclidean
KD-tree (the common workaround) is \emph{not} exact: recall collapses as
$\approx 0.5^{m}$ in the number $m$ of quaternion blocks ($0.50$ for $SO(3)$,
$0.49$ for $SE(3)$, $0.26$ for $SO(3)^2$), the signature of the double cover
$q\sim-q$ that \eqref{eq:metric} handles and a Euclidean metric ignores.
For $SO(2)$, naive recall rounds to $1.0$ only because at $N=10^6$ few
queries have their true neighbor across the $\pm\pi$ seam. Restricting queries to a band around $\pm\pi$ of ten times the mean
nearest-neighbor gap exposes it: the naive tree drops to recall $0.97$, $0.99$
and $0.88$ at $N{=}10^3,10^5,10^6$, while ours stays at $1.00$. The failure is
not rare but \emph{localized}, which is why it survives uniform-query averages
yet concentrates where revisits happen (Sec.~\ref{sec:exp-app}).

\emph{Sign convention.} The naive baseline is strongly sign-dependent: on
$SO(3)$ at $N{=}10^5$ its recall is $0.49$ with random signs but $0.99$ with
canonicalized ($q_w\ge0$) ones, canonicalization moving the seam to $q_w{=}0$
rather than removing it. Ours is insensitive in correctness and cost: recall
$1.00$ and $16.9$/$17.9$/$17.6\,\mu$s under random, canonicalized and
sign-continuous conventions (${<}6\%$ spread).

\begin{table}[t]
\centering
\caption{Manifold NN at $N=10^6$, $k=1$. Recall is against the exhaustive
ground truth under the same metric \eqref{eq:metric}.
$t_{\mathrm{KD}}$/$t_{\mathrm{BF}}$: mean query times for the manifold KD-tree
and brute force.}
\label{tab:manifold}
\renewcommand{\arraystretch}{1.02}
\begin{tabular}{l|cc|ccc}
\hline
Space & \multicolumn{2}{c|}{recall} & $t_{\mathrm{KD}}$ & $t_{\mathrm{BF}}$ & speedup\\
 & ours & naive & [$\mu$s] & [$\mu$s] & ($\times$)\\
\hline
$SO(2)$    & \textbf{1.00} & 1.00 & 3.3   & 4573 & $1399$\\
$S^2$      & \textbf{1.00} & 1.00 & 6.0   & 2936 & $493$\\
$SO(3)$    & \textbf{1.00} & 0.50 & 139   & 3826 & $27.6$\\
$SE(3)$    & \textbf{1.00} & 0.49 & 96    & 5252 & $54.7$\\
$SO(3)^2$  & \textbf{1.00} & 0.26 & 747   & 6505 & $8.7$\\
\hline
\end{tabular}
\end{table}

\emph{Efficiency.} Table~\ref{tab:manifold} reports the $N=10^6$ point; the
ordering is stable across $N\in[10^3,10^6]$. On the spaces dominating robotics
practice the tree is far faster than exhaustive search, up to $\sim\!10^3\times$
for $SO(2)$ and $\sim\!55\times$ for $SE(3)$, where the translation block prunes
strongly. Pure rotation is hardest (the double cover doubles the branching), but
the block-level bound keeps it clear of brute force: $\sim\!28\times$ for
$SO(3)$, $\sim\!9\times$ for $SO(3)^2$,
where the looser per-coordinate bound was near or below parity.
This is the false-positive cost of the separable form: on rotation-heavy
spaces it inflates query time ${\sim}30\times$ ($SO(3)$ at $N{=}10^5$:
$450\,\mu$s vs $15\,\mu$s with \eqref{eq:gate}). Without a quaternion block the
separable bound already gives the speedups above, and the correction is compiled
out. The
unit sphere $S^2$ prunes like a $3$D cloud ($\sim\!500\times$): its chordal
metric \emph{is} ambient $\mathbb{R}^3$ distance, so the naive shortcut is exact
there too.

\subsection{Comparison with prior exact-manifold libraries}
\label{sec:exp-headtohead}
The relevant baselines are the existing exact-manifold KD-tree libraries. We
compare against MPNN~\cite{yershova2007mpnn} and
nigh~\cite{ichnowski2015se3,ichnowski2020concurrent}, the maintained successor
of the Ichnowski-Alterovitz $SE(3)$ tree, on $SO(3)$ and $SE(3)$ for
$N\in[10^3,3\times10^5]$, $k=1$ (Fig.~\ref{fig:manifold-libs},
Table~\ref{tab:headtohead}).

\emph{$SO(3)$ is metric-neutral:} geodesic angle, chordal distance and the
quaternion dot product are all strictly monotone in the rotation angle, so all
three libraries induce the same ordering and cross-validate against one geodesic
ground truth; every method attains recall $1.0$ at every $N$. nigh's intrinsic
partition of $S^3$ is fastest on pure rotation ($2$--$3\,\mu$s, near flat). With
the block-level bound the present tree prunes in the same regime: ${\sim}3\times$
faster than MPNN ($15$ vs $45\,\mu$s at $N{=}10^5$) and within $6$--$9\times$ of
nigh ($29$ vs $3.3\,\mu$s at $N{=}3\times10^5$), where the per-coordinate bound
trailed by over an order of magnitude ($\sim\!450\,\mu$s).

\emph{On $SE(3)$ the two libraries do not optimize the same objective.}
Ours minimizes the squared chordal-product metric of \eqref{eq:metric},
$\lVert\Delta t\rVert^2 + w_R^2 d^2_{\mathrm{chordal}}$; nigh's
\texttt{SE3Space} is a Cartesian space whose distance is the \emph{unsquared
sum} of its factors, $\lVert\Delta t\rVert + \theta$, with $\theta$ the geodesic
rotation angle. Neither is a monotone function of the other: on uniform data the
two select the same neighbor in only $0.53$--$0.58$ of queries
($N{=}10^3$--$10^5$). Each is exact for its own metric (recall $1.00$ against a
per-metric ground truth, Table~\ref{tab:headtohead}), so this row compares two
exact solvers of two \emph{different} problems on an identical workload, and we
draw no ranking conclusion from it. With that caveat, ours answers those queries
in $4.2\,\mu$s vs nigh's $9.5\,\mu$s at $N{=}10^5$ ($16.1$ vs $17.1\,\mu$s at
$3\times10^5$) and builds ${\sim}2\times$ faster.

\begin{figure*}[t]
\centering
\import{figs}{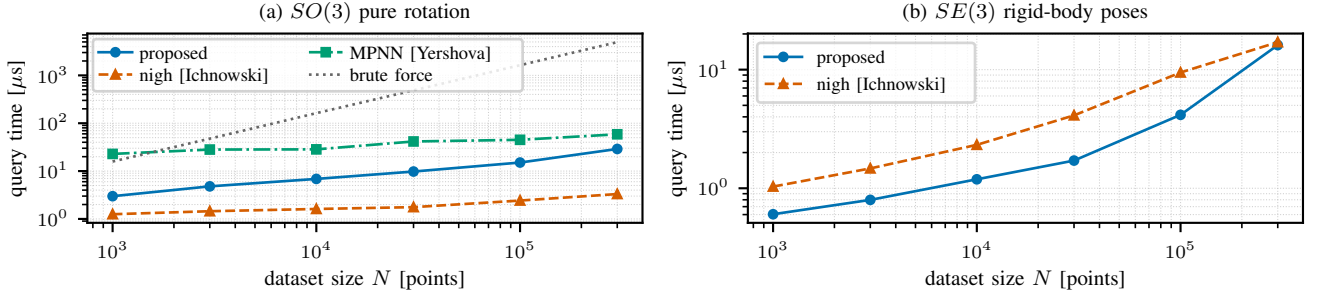}
\caption{Head-to-head vs.\ prior exact-manifold NN libraries (single thread,
$k=1$). All methods are exact for their own metric. On pure $SO(3)$,
where all three share one ground truth, nigh's intrinsic bound is
fastest, but the proposed block-level bound now beats MPNN and stays within a
small factor of nigh. The $SE(3)$ panel compares two exact solvers of
\emph{different} product metrics (see Table~\ref{tab:headtohead}) on the same
workload, so it reports cost, not ranking quality.}
\label{fig:manifold-libs}
\end{figure*}

\begin{table}[t]
\centering
\caption{Exact-manifold libraries at $N=10^5$, $k=1$ (single thread). Build is
index construction; $t_q$ is mean query time. On $SO(3)$ all three
libraries induce the same ordering, so recall is against a single \emph{shared}
geodesic ground truth. On $SE(3)$ they optimize different metrics (ours
$\lVert\Delta t\rVert^2+w_R^2 d^2_{\mathrm{chordal}}$; nigh
$\lVert\Delta t\rVert+\theta$), so recall is against each method's \emph{own}
exhaustive ground truth and the timings are not a like-for-like ranking
comparison; the two metrics agree on the NN in only ${\sim}0.55$ of queries.}
\label{tab:headtohead}
\renewcommand{\arraystretch}{1.02}
\begin{tabular}{l|ccc|ccc}
\hline
 & \multicolumn{3}{c|}{$SO(3)$ (shared GT)} & \multicolumn{3}{c}{$SE(3)$ (own GT)}\\
Library & build & $t_q$ & recall & build & $t_q$ & recall\\
 & [ms] & [$\mu$s] & & [ms] & [$\mu$s] & \\
\hline
proposed   & \textbf{14}  & 15   & 1.00 & \textbf{16} & \textbf{4.2} & 1.00\\
nigh       & 35           & \textbf{2.4} & 1.00 & 35 & 9.5 & 1.00\\
MPNN       & 195          & 45   & 1.00 & -- & -- & --\\
\hline
\end{tabular}
\end{table}

\subsection{Application: loop-closure candidate retrieval}
\label{sec:exp-app}
We next evaluate a task where $SE(3)$ NN is critical: loop-closure candidate
retrieval over SLAM keyframes, using ground-truth trajectories with revisits
from KITTI~\cite{geiger2012kitti} (00, 02, 05, 06, 08) and
MulRan~\cite{kim2020mulran} (KAIST01, DCC01, Riverside01; decimated to a
keyframe every $1$\,m). For each keyframe we retrieve the $k=8$ nearest outside a $100$-frame temporal
exclusion window, under $d^2=\lVert\Delta t\rVert^2+w_R^2 d^2_{\mathrm{chordal}}(q)$. The
rotation weight $w_R$ (meters per unit chordal distance) favors same-heading
candidates and is applied by scaling the quaternion coordinates. This is exactly the weighted metric of \eqref{eq:metric}: scaling a
quaternion block by $w_R$ multiplies its chordal contribution by $w_R^2$ and
commutes with the antipodal map, so Lemma~\ref{lem:adm} applies unchanged
(generalization (i) of its proof). The scaled block leaves the unit sphere,
which is immaterial: unit norm is needed only to read a chordal value back as an
angle. Ground truth is an exhaustive scan with the same metric and exclusion. We
compare against two workarounds common in practice: (i) a \emph{naive} Euclidean
KD-tree on the raw 7-D coordinates, with quaternions either canonicalized
($q_w\ge0$, as published) or sign-continuous (as produced by integrating
odometry); and (ii) a 3-D translation-only tree fetching $32$ candidates and
re-ranking by the full pose metric.

\begin{table}[t]
\centering
\caption{Loop-closure retrieval on KITTI and MulRan keyframes ($w_R{=}15$,
$k{=}8$, gap $100$ frames). The proposed index is exact in \emph{all} cells;
columns give the recall of the naive 7-D Euclidean tree (canonicalized /
sign-continuous quaternions) and of the 3-D-tree + re-rank workaround.}
\label{tab:loopclosure}
\renewcommand{\arraystretch}{0.98}
\scriptsize
\begin{tabular}{l r | r r | c c c}
\hline
seq & $N$ & $t_{\mathrm{BF}}$ & $t_{\mathrm{KD}}$ & naive & naive & re-rank\\
 & & [$\mu$s] & [$\mu$s] & (canon.) & (sign-c.) & \\
\hline
KITTI 00 & 4541 & 19.3 & 2.7 & 0.964 & 0.808 & 0.989\\
KITTI 02 & 4661 & 19.8 & 3.5 & 0.982 & 0.867 & 0.994\\
KITTI 05 & 2761 & 11.7 & 2.7 & 0.981 & 0.853 & 0.984\\
KITTI 06 & 1101 &  5.4 & 1.6 & 0.988 & 0.486 & 0.999\\
KITTI 08 & 4071 & 18.6 & 3.6 & 0.920 & 0.893 & 0.972\\
\hline
KAIST01      & 5877 & 22.0 & 2.3 & 0.959 & 0.698 & 0.993\\
DCC01        & 4684 & 17.8 & 1.6 & 0.981 & 0.443 & 0.995\\
Riverside01  & 6024 & 26.1 & 1.1 & 0.999 & 0.968 & 1.000\\
\hline
\end{tabular}
\end{table}

The manifold index returns the exact candidate set on every sequence, weight and
$k$, at $1.1$--$3.6\,\mu$s per query, $3$--$25\times$ faster than exhaustive scan
and only marginally slower than the (incorrect) naive tree: \emph{exactness costs
essentially nothing}. Both workarounds lose candidates. The naive tree falls to
$0.44$ (sign-continuous) and $0.92$--$0.99$ (canonicalized; canonicalization
merely moves the seam), with errors concentrated on revisited streets ($25\%$ of
KITTI 00 and $65\%$ of DCC01 queries get a wrong candidate set,
Fig.~\ref{fig:loopclosure}). The translation-only re-rank has no exactness
guarantee and degrades as orientation matters ($0.71$ recall at $w_R{=}50$ on
KITTI 08).

\begin{figure}[t]
\centering
\import{figs}{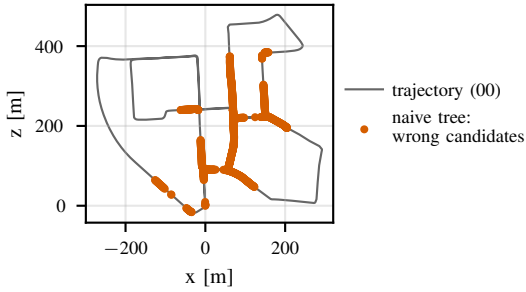}
\caption{Loop-closure retrieval, KITTI 00 top-down: markers are keyframes whose
candidate set from the naive (sign-continuous) tree is wrong ($w_R{=}15$). The
failures concentrate on the revisited streets, i.e.\ exactly where loop closures
are found. Only the proposed index is exact, at every $w_R$
(Table~\ref{tab:loopclosure}).}
\label{fig:loopclosure}
\end{figure}

\subsection{Incremental index vs.\ ikd-Tree}
\label{sec:exp-incremental}
Finally we evaluate the incremental index on a sliding-window LiDAR-mapping
workload: KITTI scans accumulate into a world map trimmed each frame to a cube
around the sensor (\texttt{removeOutsideBox}), interleaved with $k$-NN query
batches, measuring update and query latency and stored vs.\ live points against
per-frame static rebuild, the Bentley--Saxe forest and the original
ikd-Tree~\cite{cai2021ikdtree}.
Parameters: $80$ frames of KITTI seq-00, a $120$\,m keep cube advancing
$1.5$\,m/frame, $300$ queries/frame at $k{=}5$,
$(\alpha_{\mathrm{bal}},\alpha_{\mathrm{del}})=(0.85,0.5)$. ikd-Tree spawns its
rebuild thread unconditionally in its constructor (not switchable off), so it is
measured \emph{with} parallel rebuilding, using the FAST-LIO2 defaults
$(\delta_{\mathrm{del}},\alpha,\ell_{\mathrm{box}})=(0.5,0.7,0.2)$. The proposed
\emph{synchronous} index thus uses strictly fewer threads than its baseline
(one, no helper); the asynchronous variant uses two.
Fig.~\ref{fig:incremental} reports results on a $4.8$\,M-point persistent map
($\sim$$122$\,k points/scan). The synchronous tree achieves the best median update among bounded-memory exact
methods ($93$\,ms; ikd-Tree $261$\,ms, per-frame rebuild $605$\,ms) at
single-tree query speed ($7.3\,\mu$s vs.\ the forest's $17.5\,\mu$s); its
weakness is the update tail (p95 $830$\,ms, the near-root rebuild), closed by
the asynchronous variant (p95 $110$\,ms), which beats ikd-Tree on median, tail
and query at once. Memory stays bounded ($1.3\times$ live points, $241$\,B per live point vs
ikd-Tree's $380$), whereas the forest grows with \emph{total} inserts
($2\times$ and rising).

\emph{Ablation.} Disabling one element at a time attributes this cleanly.
Box-wise deletion is what makes the update fast: without it, trimming must test
every live point and the median rises from $93$\,ms to $12.0$\,s ($129\times$).
Partial rebuilds instead \emph{cost} update latency and buy the rest: without
them the median is unchanged ($95$\,ms) and the tail even shortens (p95
$125$\,ms), but queries slow $2.4\times$ ($5.0$ vs $2.1$\,ms) and memory grows
to $295$\,B/point. The pooled free-list is latency-neutral ($94$\,ms) but worth
$4.2\times$ in footprint ($1005$ vs $241$\,B/point). The asynchronous variant
supplies essentially all the tail gain ($7.5\times$) at $1.9\times$ memory.
Since ikd-Tree also trims box-wise, its remaining $2.8\times$ median gap is a
constant factor, not an algorithmic one: the ${\sim}122$\,k insertions per frame
dominate, and we attribute it mostly to smaller, pool-allocated nodes ($241$ vs
$380$\,B/point, hence fewer cache misses per descent) and a lock-free
synchronous update path, though these were not ablated individually across the
two code bases.

\emph{Incremental \emph{and} manifold.} The two compose, as in RRT/RRT$^*$ over
a configuration tree growing by thousands of poses. On an $SE(2)$ RRT$^*$ run
($2\times10^4$ nodes, all structures grown over one identical node sequence) the
index answers every nearest and near-set query and insertion in $0.19$\,s
against $20.3$\,s for rebuilding a static tree each step ($\sim\!108\times$),
both exact, while a Euclidean tree on raw $(x,y,\theta)$ returns a non-nearest
node for $5.7\%$ of queries. nigh, which also inserts incrementally
here, finishes in $0.11$\,s ($1.7\times$ faster than ours) under its own
Cartesian metric $\lVert\Delta t\rVert+|\Delta\theta|$, which agrees with
\eqref{eq:metric} on only $0.83$ of these queries; each is exact for its own
metric. On this insert-only workload a planner-specialized structure is faster;
our deletion and box trimming under churn are never exercised here.

\begin{figure}[t]
\centering
\import{figs}{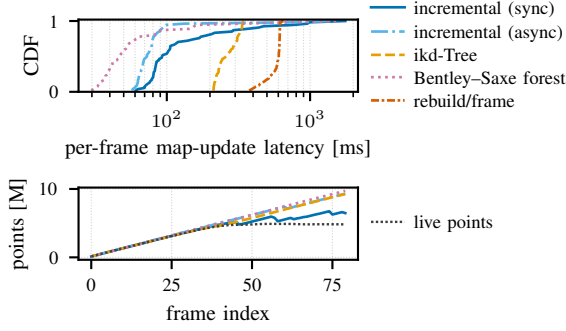}
\caption{Incremental index vs.\ alternatives on the sliding-window KITTI
workload ($4.8$\,M-point map, $k{=}5$; single thread except the async variant's
background rebuild and ikd-Tree's, which is always on). Top: CDF of
per-frame update latency (steady state). Bottom: physically stored points; the
synchronous tree stays within $1.3\times$ the live set, while the Bentley--Saxe
forest never reclaims trimmed points.}
\label{fig:incremental}
\end{figure}

\section{CONCLUSIONS}
\label{sec:apps}

We presented a header-only, dependency-free KD-tree library performing
\emph{exact} NN search on user-defined product manifolds declared at compile
time. The base library is already a de-facto building block in open-source robotics
(e.g.\ SLAM~Toolbox~\cite{Macenski2021}); the extension lets the same deployed
tree serve queries on a group: pose/orientation retrieval, orientation-aware
place recognition, and nearest-configuration
planning~\cite{lavalle2001rrt,karaman2011rrtstar}. Exactness is w.r.t.\ the product metric \eqref{eq:metric}, whose rotational
factor is chordal; for a single rotation or sphere factor that ordering is the
geodesic one. The key enabler is an admissible point-to-interval lower bound
(Lemma~\ref{lem:adm}), wrap-aware for circular coordinates and
double-cover-aware for quaternion blocks (tightened by the block-level chordal
bound of~\cite{yershova2007mpnn}, applied here at no extra storage), so
pruning stays valid on $SO(2)$, $SO(3)$, $S^n$, $SE(2)$, $SE(3)$, tori and their
products, in both single-tree indices.
Experiments confirm exactness where the common Euclidean workaround is provably
wrong, large speedups on the spaces most relevant to robotics, and that on real
trajectories the workarounds lose precisely the loop closures the queries exist
to find. Unlike prior
planning-specific trees~\cite{yershova2007mpnn,ichnowski2015se3}, the
abstraction is compile-time composable, admits base spaces they do not
($S^n$, hence $\mathbb{R}^3\times S^2$), and adds no per-node storage. The release
also adds an incremental self-balancing index for sliding-window LiDAR maps that
beats the ikd-Tree on update latency at single-tree query speed and bounded
memory, accepting the same manifold metrics under churn. 

\section*{ACKNOWLEDGMENT}
This work was partly supported by the Google Open Source Programs Office (OSPO)
by means of a GSOC2017 grant to Pranjal Rai.

\bibliographystyle{IEEEtran}
\bibliography{references}

\end{document}